\documentclass[conference]{IEEEtran}
\usepackage{cite}
\usepackage{amsmath,amssymb,amsfonts,mathtools,amsthm}
\usepackage{algorithmic}
\usepackage{graphicx}
\usepackage{textcomp}
\usepackage{xcolor}

\usepackage[english]{babel}
\usepackage[square,numbers]{natbib}
\newtheorem{theorem}{Theorem}[section]

\newtheorem{lemma}[theorem]{Lemma}
\newtheorem{corollary}[theorem]{Corollary}
\newtheorem{definition}[theorem]{Definition}

\newcommand{\C}[2][\{y,y^{\prime}\}]{C(#1,#2)}
\newcommand{\Ce}{C(\cdot)}
\newcommand{\yp}{y^{\prime}}
\newcommand{\ypp}{y^{\prime\prime}}
\newcommand{\spr}{s^{\prime}}

\newcommand{\zpr}{z^{\prime}}
\newcommand{\z}[2][j]{z_{#1,d(#2)}}
\newcommand{\zp}[2][j]{z^{\prime}_{#1,d(#2)}}
\newcommand{\zj}[3][j]{z^{#3}_{#1,d(#2)}}
\newcommand{\zline}[5][j]{z_{#1,d(#3)}^{#2} &> z_{#1,d(#4)}^{#2} > z_{#1,d(#5)}^{#2}}

\newcommand{\zanysec}[5][j]{z_{#1,d(#3)}^{#2} &> z_{#1,d(#4)}^{#2} \backslash z_{#1,d(#5)}^{#2}}
\newcommand{\argmax}{\operatornamewithlimits{argmax}}

\begin{document}

\title{Prediction Instability in\\Machine Learning Ensembles}

\author{\IEEEauthorblockN{1\textsuperscript{st} Jeremy Kedziora}
\IEEEauthorblockA{\textit{Department of Computer Science and Software Engineering} \\
\textit{Milwaukee School of Engineering}\\
Milwaukee, WI USA \\
kedziora@msoe.edu}
}

\maketitle

\begin{abstract}
In machine learning ensembles predictions from multiple models are aggregated. Despite widespread use and strong performance of ensembles in applied problems little is known about the mathematical properties of aggregating models and associated consequences for safe, explainable use of such models. In this paper we prove a theorem that shows that any ensemble will exhibit at least one of the following forms of prediction instability. It will either ignore agreement among all underlying models, change its mind when none of the underlying models have done so, or be manipulable through inclusion or exclusion of options it would never actually predict. As a consequence, ensemble aggregation procedures will always need to balance the benefits of information use against the risk of these prediction instabilities. This analysis also sheds light on what specific forms of prediction instability to expect from particular ensemble algorithms; for example popular tree ensembles like random forest, or xgboost will violate basic, intuitive fairness properties.  Finally, we show that this can be ameliorated by using consistent models in asymptotic conditions.
\end{abstract}


\section{Introduction}
\noindent The promise of machine learning and artificial intelligence is to automate decision making in an efficient, rational form that makes optimal use of information and offers consistency of choice.  A popular strategy to optimize machine learning-based decision making is to maximize predictive accuracy by aggregating the insights of multiple models together, a process called ensembling.  Decades of research and experimentation have not revealed an obvious ``best" method for aggregating models, an outcome mirroring no-free-lunch theorems. Moreover, despite widespread use and strong performance of ensembles in applied problems little is known about the mathematical properties of aggregating models and associated consequences for safe, explainable use of such models. An alternative to focusing solely on maximizing predictive accuracy would be to proceed axiomatically by setting down desirable properties in an aggregation procedure and then characterizing the set of aggregation rules that satisfies them. 

To motivate the axiomatic approach, consider a tree ensemble used for classification trained via e.g. the random forest algorithm \citep{breiman-1996,breiman-1999,breiman-2001} or the xgboost algorithm \citep{chen:2016}--a collection of $m$ decision trees with tree $j$ represented as a pair $(T_j,L_j)$ where $T_j$ is a partition of the data and $L_j(x)$ is a function mapping from an observation in the feature space $x$ to an element of the partition. One method to use this model to make predictions or choices would be to aggregate scores from the $m$ trees and choose the class $y$ with the largest such score:
\begin{align*}
\argmax_{y}\left\{\frac{1}{m}\sum_{j=1}^m\left(\frac{1}{|L_j(x)|}\sum_{y^{\prime}\in L_j(x)}I(y=y^{\prime})\right)\right\}
\end{align*}
where $I(\cdot)$ is the indicator function. As an example, suppose that we fit such an ensemble on a data set with three classes and then, for a pair of points in the feature space $x_1$ and $x_2$, conduct inference to obtain the results in Table \ref{table:example}.

\begin{table}[htbp]
\caption{}{}
Scores for $x_1$:
\begin{center}
\begin{tabular}{ccccc}
Class & tree 1 & tree 2 & tree 3 & Agg. Score\\
\hline
\hline
1 & 0.40 & 0.40 & 0.40 & 0.4000\\
2 & 0.34 & 0.35 & 0.5 & 0.3967\\
3 & 0.26 & 0.25 & 0.1 & 0.2033\\
\end{tabular}
\end{center}
Scores for $x_2$:
\begin{center}
\begin{tabular}{ccccc}
Class & tree 1 & tree 2 & tree 3 & Agg. Score\\
\hline
\hline
1 & 0.40 & 0.40 & 0.40 & 0.4000\\
2 & 0.36 & 0.35 & 0.5 & 0.4033\\
3 & 0.24 & 0.25 & 0.1 & 0.1967\\
\end{tabular}
\label{table:example}
\end{center}
\end{table}

\noindent From Table \ref{table:example} we would predict or choose class $1$ for $x_1$ and class 2 for $x_2$.  But notice that these two points in feature space are similar in the sense that trees 2 and 3 have identical scores, and tree 1 yields that class 1 is more likely than class 2, which is more likely than class 3 in both cases.  Despite this, after aggregation the ensemble prediction or choice is not the same for these two examples.  That is, the ensemble ``changed its mind" about the relative likelihood of the classes \textit{even though no individual model did}.

In this note we argue that this sort of prediction or decision-making inconsistency is an intrinsic part of ensembling.  Specifically, we prove a theorem that shows that any ensemble will exhibit at least one of the following forms of prediction instability. It will either ignore agreement among all underlying models, change its mind when none of the underlying models have done so, or be manipulable through inclusion or exclusion of options it would never actually predict.  As a consequence, ensemble aggregation procedures will always need to balance the benefits of information use against the risk of choice instability.  This result is driven by the capacity of the individual models within the ensemble and so provides an additional axiomatic rationale for the use of weak learners within an ensemble.  We also show that this prediction instability is eliminated when using consistent models in asymptotic conditions and so can be thought of as a finite-sample property of ensembling.

The remainder of this section lays out other work that we draw upon, the second section introduces our framework and axioms, the third section specifies our result on ensemble prediction instability, the fourth section discusses implications of these results, and the fifth section establishes our asymptotic result.

\subsection{Related Work}
\noindent The argument for ensembling many models together can be traced back to the jury theorems inspired by \citet{condorcet-1785}.  These methods are heavily used in classical supervised learning \citep[e.g.][]{breiman-2001}, but also appear in modern deep learning \citep{ganaie-2022}, in the token generation of large language models \citep[e.g.][]{shazeer-2017}, and in reinforcement learning \citep{song-2023}. 

Modern use of ensembling techniques tends to involve choosing a base algorithm to construct learners, sampling data/features to train that base algorithm via bagging \citep{breiman-1996,breiman-1999,breiman-2001} or boosting \citep{freund-1997} numerous times, and finally aggregating the outputs of individual models, e.g via hard or soft voting, or some learned procedure, e.g. stacking \citep{wolpert-1992}.\footnote{See \citep{dietterich-2000}, \citep{maclin-2011}, \citep{polikar-2006}, or \citep{rokach-2010} for reviews.}  Despite the availability of numerous options for aggregation little scholarship has focused on studying the qualities of different aggregation procedures.  As others \citep[e.g.][]{wolpert-1997,werbin-ofir-2019} have noted, the optimal choice of aggregator is likely to be problem-specific.  Given this, we study aggregators from an axiomatic perspective.  This allows us to identify combinations of desirable properties, seek the  aggregators that satisfy them, and thereby characterize the consistency and stability of the predictions or decision-making derived from ensembles.

Social choice theory has long analyzed aggregative mechanisms axiomatically in the context of economics and political science.  One of the earliest social choice results, Arrow's impossibility theorem argued that no method of aggregating the ordinal preferences of a set of individuals into a single group-level ordinal preference could satisfy a set of desirable properties chosen to maximize consistency in decision-making and democratic principles \citep[see][]{arrow-1951}.  

We take advantage of three specific extensions to Arrow's original result to characterize ensemble aggregation.\footnote{For a review see \citep{patty-2019}.}    First, Arrow analyzed an orders-to-order mapping and machine learning models typically output scores, implying a scores-to-score mapping.  \citep{harsanyi-1955}, \citep{samuelson-1967}, \citep{sen-1970a}, \citep{schwartz-1970}, \citep{fishburn-1972}, \citep{kalai-1977} and others extended Arrow's result to cardinal preferences and we utilize their insights here.  Second, in operational contexts machine learning is often used to make decisions, implying a focus on the choice derived from the scores-to-score mapping rather than the output score itself.  This mirrors efforts to analyze the compatibility of Arrow's axiomatic approach with so-called social choice functions that identify a subset of the alternatives rather than a preference ordering over them \citep[e.g.][]{sen-1969, sen-1970b, sen-1971} and that inform our own formalization.  Finally, efforts have been made to show that Arrow's axioms are special cases of other properties, leading to more general impossibility theorems.  We will make use of \citep{eliaz-2004} in this regard.

\section{Framework}
\noindent We begin by formalizing the notion of an ensemble.  Consider a variable that an analyst wishes to predict with possible values $y$ given by a set $Y$ so that $y \in Y$.  We will refer to this variable as the \textit{response} and its values as \textit{labels}.  If it is a regression problem then $Y\subseteq\mathbb{R}$.  If the goal is to model a count then $Y$ is countably infinite.  If it is a classification problem then $Y$ is finite.  

With this response variable we will associate a set of predictor variables referred to as \textit{features} and assumed to be an element $x$ of some feature data space $X$.  Given a data set of examples $D_n = \{(x_1,y_1),\hdots,(x_n,y_n)\}$ training a model $j$ is equivalent to choosing a scoring function $s_j:X\to S_j(Y)\subseteq\mathbb{R}^{|Y|}$ where $s_j(x) = \{s_j(x,y)\}_{y\in Y}$.  When we need to express the dependence of $s_j(\cdot)$ on the training data $D_n$ we will write $s_j(x|D_n)$; otherwise, we will simply write $s_j(x)$.  This model scoring function maps the feature inputs to scores that can be used to predict values of $y$ correctly, according to some well-defined measure of correctness.  In analogue with our motivating tree ensemble example, a model scoring function could be the output of a single decision tree.

\subsection{Model Ensembles}
\noindent The intuition of ensembling is that it should be possible to leverage the complementary qualities and diversity of different models to create a more robust and accurate representation of $y$.  Thus, we suppose that we have a set of $m\geq 2$ models, each with its own scoring function learned from the data indexed by $j$, and each mapping into $S_j(Y)\subseteq\mathbb{R}^{|Y|}$ so that:
\begin{align*}
s(x) = &\left\{\left\{s_j(x,y)\right\}_{y\in Y}\right\}_{j=1,\hdots,m}\in \prod_jS_j(Y) \equiv S(Y).
\end{align*}
We will write $s(x)$ to mean the function mapping from $X$ to $S(Y)$ as defined above.  When we want to work with an arbitrary element of $S(Y)$ rather than the function we will write it as $z = \{\{z_{j,d(y)}\}_{y\in Y}\}_{j=1,\hdots,m}$ where $d(y)$ is the dimension measuring the score for $y$, and where $s(x) = z$ if and only if $s_j(x,y) = z_{j,d(y)}$ for all $j$ and all $y\in Y$.

With a set of model scoring functions in hand we will model the ensemble prediction as a choice from the set of label values, $Y$.  Accordingly, we formalize the method for using the scores from many models to make predictions or choices in the following way.\footnote{We exclude $\emptyset$ from the range of $\Ce$ to focus throughout on ensemble choice aggregators that actually make choices.  Later, we will use this to rule out a form of intransitivity in Lemma \ref{DC_IC_transitivity}.}

\begin{definition} An \textbf{ensemble choice aggregator} is a set-valued function $C:2^Y\backslash\{\emptyset\}\times S(Y) \rightarrow 2^Y\backslash\{\emptyset\}$ such that for any $Y^*\subseteq Y$ and any $z\in S(Y)$ we have $C(Y^*,z)\subseteq Y^*$ and also that $y\in\C[Y^*]{z}$ if there does not exist $\yp\in Y^*$ such that $\{\yp\} = \C{z}$.
\end{definition}
The job of this function is to use the model scores to make a choice of one or more options from a subset of $Y$.\footnote{The ensemble choice aggregator is purposefully abstract and general to maximize the reach of our theoretical analysis.  A common, specific ensemble choice aggregator would be to choose the response label that maximizes the average score across models given the input.  We will analyze this below.} Finally, an ensemble is a pair $(C,s)$ consisting of the set of individual model scoring functions and the method of aggregating them.

\subsection{Model Capacity Properties}
\noindent The representational capacity of the individual models in an ensemble will turn out to be important in our results.  We will use two properties to capture model capacity. First, the extent to which scores consistent with each possible ordering of the response labels are available in the range of an individual model bounds the flexibility of that model; each potential ordering of response labels excluded from $S_j(Y)$ would prevent the model from learning that ordering no matter what information is available in the data.  

To formalize this idea, let $\mathbf{R}$ denote the set of all complete and transitive binary relations on $Y$; any of these binary relations will induce an ordering of the response labels in $Y$.  Given a particular weak ordering $R\in \mathbf{R}$ we can collect all the different scores consistent with this ordering as a set:
\begin{align*}
W(Y,R) = \left\{z\in\mathbb{R}^{|Y|}\left|\mbox{for all }y,y^{\prime}\in Y \atop{y\hspace{1mm} R\hspace{1mm}y^{\prime}\mbox{ iff }z_{d(y)}\geq z_{d(y^{\prime})}}\right.\right\}
\end{align*} 

A potentially high capacity model would then be one where no ordering of label values is excluded from the model's range.  Formally:

\begin{definition}Say that model $j$ is an \textbf{unconstrained learner} if for all $R\in\mathbf{R}$ we have $S_j(Y)\bigcap W(Y,R)\neq\emptyset$.\label{def:unconstrained_learner}
\end{definition}
An unconstrained learner is one in which scores could be learned that correspond to any ordering of the response labels.

Second, we can also be rigorous about exactly how many of these orderings a model actually learns.  Formally:

\begin{definition}Given $s(x)$ as defined above say that $S^*\subseteq S(Y)$ is \textbf{observable by $\mathbf{s(x)}$} if for all $z\in S^*$ there exists $x\in X$ such that $s(x) = z$.\label{def:observability}
\end{definition}
Under this definition, a given ordering over the possible label values is learned by a model if there exists a point in the feature space that maps to scores that correspond to it.\footnote{Note neither of these definitions are the same as assuming surjectivity of $s_j(\cdot)$.  We will have more to say about how this relates to our results below.}

\subsection{Aggregator Properties}
\noindent Our next task is to specify properties that seem like \textit{a priori} desirable characteristics in an ensemble.  Our goal will then be to assess the extent to which it is possible to have an ensemble score aggregator that satisfies them.  We will focus on two ways to assess choice consistency:
\begin{enumerate}
\item How do the predictions or choices made by an ensemble respond to post-training changes to the set of labels the response could take on? 
\item How the predictions or choices made by an ensemble respond when we impose requirements on how the score aggregator uses information from the underlying models?
\end{enumerate}

Change to the set of possible labels comes in two forms: \textit{deletions} and \textit{insertions}.  For example, an online retailer using a recommender system could choose to permanently drop products from its catalogue; if it is expensive to retrain models then deleting a label option would be a temporary method to ensure that recommendations remain in sync with available products.  We propose the following definition--adapted to our context from \citep{arrow-1959} and \citep{sen-1969}--to formalize ensemble decision consistency given these types of changes:

\begin{definition}
Say that $\Ce$ is \textbf{insertion/deletion consistent} if, for all $z\in S(Y)$ and all $Y_1^{*}, Y_2^{*}\subseteq Y$ such that $Y_1^*\subseteq Y_2^{*}$, we have that $Y_1^*\bigcap \C[Y_2^*]{z}\neq\emptyset$ implies that $\C[Y_1^*]{z} = Y_1^*\bigcap \C[Y_2^*]{z}$.  \label{ICDC}
\end{definition}
\noindent Insertion/deletion consistency requires that if the ensemble chooses or predicts some label options when all possible labels for the response are available and then the set of labels is reduced but still contains some previously chosen labels then 1) no label not chosen before now becomes chosen and 2) no previously chosen label is now not chosen.  In other words, the ensemble cannot be compelled to change what it does choose via requiring that it ignore label options it never chose in the first place; adding new label options cannot break ties.\footnote{\citep{sen-1969} gives the following examples of this property for intuition on page 384 of his original paper: ``...if the world champion in some game is a Pakistani, then he must also be the champion in Pakistan...if some Pakistani is a world champion, then all champions of Pakistan must be champions of the world."}

We turn next to consider desirable requirements on how an ensemble uses information, adapted from \citet{arrow-1951}, \citet{sen-1970a}, and \citet{eliaz-2004}.  
\begin{definition}
Say that $\Ce$ is a \textbf{nondegenerate ensemble choice aggregator} if there does not exist $j$ such that, for all $y,y^{\prime}\in Y$ and all $z\in S(Y)$, we have that $z_{j,d(y)}>z_{j,d(\yp)}$ implies that $\{y\}= \C{z}$.\label{NECA}
\end{definition}

\begin{definition}
Say that $\Ce$ satisfies \textbf{ensemble unanimity} if, for all $y,y^{\prime}\in Y$ and all $z\in S(Y)$, we have that $z_{j,d(y)}> z_{j,d(\yp)}$ for all $j$ implies $\{y\} = \C{z}$.\label{score_unanimity}
\end{definition}

\begin{definition}
Say that $\Ce$ \textbf{respects model choice reversal} if, for all $y,y^{\prime}\in Y^*\subseteq Y$ and all $z,z^{\prime}\in S(Y)$, we have that $y\in \C{z}$, $y^{\prime}\notin \C{z}$, and $y^{\prime}\in \C{z^{\prime}}$ implies that there exists $j$ such that $z_{j,d(y)}>z_{j,d(\yp)}$ and $z^{\prime}_{j,d(y)}<z^{\prime}_{j,d(\yp)}$.\label{CR}
\end{definition}

\noindent First, if the intuition of ensembling--that many models improve on decision-making over a single model--is born out then the ensemble should use the information of multiple models to produce its decisions and the aggregator should not be an isomorphism of any single model.  Definition \ref{NECA} is a minimal version of such a requirement in that it requires that a single model does not completely determine the final choice.  Second, it seems natural to assume that if every model in the ensemble scores $y$ higher than $y^{\prime}$ then the choice of the ensemble should reflect this; we formalize this in Definition \ref{score_unanimity}.  Third, how should the propensity of the ensemble to change its decision between labels depend on the propensity of the underlying models to do so?  In Definition \ref{CR} we stipulate that the ensemble cannot change its mind for $y$ and $\yp$ unless at least one model in the ensemble changed its mind for $y$ and $y^{\prime}$.\footnote{Model choice reversal is a generalization of the independence of irrelevant alternatives property used by \citep{arrow-1951} and many others.  See \citep{eliaz-2004} Propositions 1 and 2 for analysis.} 

\section{An Ensemble Aggregated--Prediction Instability Theorem}
\noindent In this section we will argue that there is a fundamental tradeoff between information use and stability of ensemble predictions and choices in the sense that no aggregation procedure satisfies defintions \ref{ICDC}, \ref{NECA}, \ref{score_unanimity}, and \ref{CR}.  Our focus here will be on building intuition for why our theorem is true via sketching the arguments for it.  The following section will expand on implications with examples and discussion.  Full proofs for all results are in the appendix.   

Transitivity plays a central role in the proof of our result:

\begin{definition}
Say that $\Ce$ satisfies transitivity if, for all $y,\yp,\ypp\in Y$ and any $z\in S(Y)$, we have that $y\in\C{z}$ and $\yp\in\C[\{\yp,\ypp\}]{z}$ implies $y\in\C[\{y,\ypp\}]{z}$.\label{def_of_trans}
\end{definition}
\noindent We being by assuming that all of our conditions hold.  In the first step of our argument we show that any ensemble choice aggregator that satisfies insertion/deletion consistency will be transitive, formalized in this Lemma:

\begin{lemma} $\Ce$ satisfies insertion/deletion consistency if and only if it is transitive.\label{DC_IC_transitivity}
\end{lemma}


\noindent So, given an ensemble choice aggregator that is insertion/deletion consistent, and thus transitive, the next step is to argue that transitivity implies that at least one information use requirement must be violated.  The strategy to show this follows insights from \citet{geanakoplos-2005} and \citet{fey-2014}.  We use ensemble unanimity to identify a set of scores in which there is a \textit{critical model} -- a candidate for a model in the ensemble whose scores might completely determine ensemble choices.  The second part of the argument is to apply ensemble unanimity, transitivity, and model choice reversal to a series of scores derived from making ``small" changes to the original set of scores used to identify the critical model.  In particular, where model choice reversal holds, if no model reverses its scores for one label versus another then the ensemble cannot change its choices between the two labels.  Formally: 

\begin{lemma}Suppose that $\Ce$ respects model choice reversal.  Then $z_{j,d(y)}> z_{j,d(\yp)}$ if and only if $z^{\prime}_{j,d(y)}> z^{\prime}_{j,d(\yp)}$ for all $j$ implies that $\C{z} = \C{z^{\prime}}$.\label{mcr_same_choice}
\end{lemma}
\noindent Lemma \ref{mcr_same_choice} allows us to show that each of the series of model scores must yield the same choice as a set of scores in which the critical model determines the choice; we conclude that $\Ce$ must not be a nondegenerate ensemble choice aggregator, a contradiction that establishes that there must be a set of scores on which at least one of our definitions is violated.  Formally:

\begin{theorem} (Aggregated--Prediction Instability) Suppose that $|Y|\geq3$, that all individual models are unconstrained learners, and that $\Ce$ is a non-degenerate ensemble choice aggregator on $S(Y)$.  Then:
\begin{enumerate}

\item There exists $S^*\subseteq S(Y)$ such that $\Ce$ violates at least one of insertion/deletion consistency, ensemble unanimity, or model choice reversal on $S^*$.

\item If $s(x)$ is surjective then $S^*$ is observable by $s(x)$.

\end{enumerate}

\label{thm:aggregation_instability}
\end{theorem}

\section{Implications and Discussion}
\noindent Next we turn to laying out what this theoretical result means for how we should expect machine learning ensembles to behave in practical settings and what lessons we take away.

\subsection{Aggregation creates problems}
\noindent In general, theorem \ref{thm:aggregation_instability} tells us that if we would like to have a machine learning ensemble use the insights of multiple models in making predictions and choices (definition \ref{NECA}) then we must accept the risk that the ensemble makes those predictions and choices in a counterintuitive or unstable fashion at least some of the time, ignoring model agreement (definition \ref{score_unanimity}), changing its mind in a way that the underlying models cannot justify (definition \ref{CR}) or in response to the presence of label options it would not predict or choose (definition \ref{ICDC}).

Failures of these properties are forms of ``instability" in the sense that they generate sudden changes in the output of the ensemble.  Violations of insertion/deletion consistency are a form of discontinuity in the scoring space in the sense that changes to the unchosen portion of the set of response labels can induce changes in the ensemble prediction or choice even if the individual model scores are held constant.  Violations of ensemble unanimity are also a form of discontinuity in the scoring space in that more model agreement on the ordering of the response labels can lead to the ensemble ignoring labels which all agree to be more or most likely.  Violations of model choice reversal are a form of discontinuity in the underlying information in the sense that two scores can both be consistent with the same ordering of the set of label values for all models and yet result in different ensemble predictions or choices.

Finally, we note that part 1) of theorem \ref{thm:aggregation_instability} implies that violations of these properties are intrinsic characteristics of the method by which the individual models in the ensemble are aggregated, $\Ce$, and do not depend on what $s(x)$ has learned.  Choosing a particular aggregation method determines which properties will be violated and where in the scoring space those violations will occur.

\subsection{Capacity of individual models makes problems observable}
\noindent Violations of these properties--and particularly the extent to which they will actually be encountered by a user during model inference--are also tied to the capacity of the individual models within the ensemble to represent the relationship between features and labels.  In our setting we formalized model capacity in terms of two properties: unconstrained learners and observability.  First, the range of the scoring function, $S_j(Y)$, bounds the ability of the model to encode information about the label.  When the models in the ensemble are unconstrained learners this range includes scores consistent with all orderings of the response labels, so that any ordering can be learned.  This is sufficient to guarantee that there will be learnable regions of the scoring space where aggregation will lead to violations of our properties.

Second, if each vector of scores in $S(Y)$ is observable in the sense that it is mapped to from some point in the feature space then we can conclude that the model has enough capacity to represent any relationship between features and response labels.  If $s(x)$ is surjective this is sufficient to guarantee that all scores in $S(Y)$ are observable and could be encountered during inference.  Consequently, it is also sufficient to guarantee that violations of our properties are observable in the sense that they can be encountered at inference because real points in feature space lead to the scores associated with those violations.
 
High capacity individual models increase the risk of observable violations of at least one of the properties analyzed here.  In this sense, theorem \ref{thm:aggregation_instability} provides an additional axiomatic rationale for the use of weak learners--limited capacity individual models--within an ensemble.

\subsection{Popular aggregation procedures violate model choice reversal}
\noindent Specifying the aggregation method as a non-degenerate ensemble choice aggregator that satisfies ensemble unanimity is relatively undemanding and it is easy to write down score aggregators that will use the input of more than one model and respect agreement among all models; indeed a simple weighted average where all weights are positive works.  Model choice reversal--the requirement that the ensemble cannot change its mind unless some of the underlying models do so--is more difficult to guarantee, and multiple popular approaches to ensembling will violate it.
 
We expand on this with some examples.  A natural, perhaps even ubiquitous, way to use an ensemble to make decisions about response labels would be to aggregate the scores of individual models and then require that the ensemble choose the label that has the highest aggregated score.  This is clearly insertion/deletion consistent and transitive under any method of score aggregation.

In a voting ensemble individual model scores are combined to make a final prediction or decision through either hard or soft voting.  In hard voting the voting ensemble predicts the class label that receives the plurality of votes from the constituent models in which case:
\begin{align}
&\C[Y^*]{s(x)}\nonumber\\
&\hspace{3mm} = \argmax_{y\in Y^*}\left\{\sum_j I\left(y \in \argmax\left\{s_j(x,y)\right\}\right)\right\}.\label{hard_voting}
\end{align}
In soft voting the ensemble averages the scores assigned to each response label value by the models in which case:
\begin{align}
\C[Y^*]{s(x)} = \argmax_{y\in Y^*}\left\{\frac{1}{m}\sum_{j=1}^ms_j(x,y)\right\}.\label{soft_voting}
\end{align} 
Either way, $\Ce$ is a nondegenerate ensemble choice aggregator and satisfies unanimity, and so theorem \ref{thm:aggregation_instability} implies a violation of model choice reversal in multiclass or regression problems.  

These voting procedures are commonly used to aggregate individual models.  For example, tree ensembles, including those generated by e.g. random forest, XGBoost or Light GBM, often apply soft voting or weighted soft voting for aggregation.  Even reinforcement learning methods can apply soft voting; in \citep[e.g.][]{xu-2019} multiple agents work together in smart collaboration to maximize rewards via voting.  As a simple example, in $Q$-learning, agents learn to encode the long-term value for taking actions in a function mapping from an action $a$ and state of the world $\sigma$, the so-called $Q$-value, $Q(\sigma,a)$.  If each agent learns its own $Q$-value, these values are aggregated via soft voting, and the action associated with the largest is chosen then:
\begin{align}
\C[Y^*]{Q(\sigma)} = \argmax_{a\in Y^*}\left\{\frac{1}{m}\sum_{j=1}^mQ_j(\sigma,a)\right\}.\label{RL_soft_voting}
\end{align} 

We summarize these observations in the following corollary:
\begin{corollary} If $\Ce$ satisfies equations \ref{hard_voting}, \ref{soft_voting}, or \ref{RL_soft_voting} then there exists $S^*\subseteq S(Y)$ such that it violates model choice reversal on $S^*$.
\end{corollary}
The use of a unit vote for the most popular response label or action, \citep[e.g.][]{breiman-2001}, implemented as a form of soft voting implies a violation of model choice reversal.  The decision making inconsistency that we used to motivate our argument is indeed ubiquitous.

\subsection{Violations of model choice reversal complicate explainability}
\noindent As we discussed above, model choice reversal is a form of smoothness.  When model choice reversal holds ``small" changes in individual model scores--no modeling changing its mind--cannot provoke ``large" changes in aggregated decisions.  On the other hand when model choice reversal is violated small changes to individual model scores will generate large changes to the outputs in at least some part of the scoring space.  Small changes to individual model scores can come about in multiple ways, but two obvious pathways are 1) through training updates to the models and 2) through conducting inference on different points in feature space.  In the case of training updates, violations to model choice reversal mean that the ensemble cannot guarantee similar predictions or choices by the ensemble before and after the training update even if that update yields very minimal change to the models.  In the case of inference for different points in the feature space the ensemble cannot guarantee that two nearly identical points in feature space will generate similar decisions even when the individual models are locally smooth.  We suspect that this will manifest as unpredictable and unintuitive behavior that complicates the explainability of an ensemble and compromises the user experience for a non-expert.  

\subsection{Ensembles cannot learn their way out}
\noindent Learning the ensemble choice aggregator along with the individual models will simply impose a particular violation as an accidental by-product of optimizing the loss function.  As an example of this, in stacked generalization we train a collection of models on a data set and then learn an aggregation method that assembles the output of these models into a single prediction as a secondary learning task \citep{wolpert-1992}.  One approach to this would be to learn a set of weights that could be used to aggregate the individual models.  For example, in the case where we have a set of Bayesian regression models combined via learned weights the posterior distribution over the parameters of each model induces a score for each $y\in Y$ when given $x$ in the form of the model evidence and the ensemble choice aggregator is:
\begin{align*}
\C[Y^*]{s(x)} = \argmax_{y\in Y^*}\left\{\sum_jw_js_j(x,y)\right\}
\end{align*}
where $w_j$ is the learned weight for model $j$. If $w_j<0$ for some $j$ then the ensemble choice aggregator will violate ensemble unanimity.  If $w_j=0$ for all but one model then this is not a nondegenerate ensemble choice aggregator, for example as a consequence of using a model specification method, e.g. the adaptive lasso \citep[see][]{zou-2006}. If $w_j>0$ for all $j$ then it is a nondegenerate ensemble choice aggregator and satisfies ensemble unanimity, and so by theorem \ref{thm:aggregation_instability} is guaranteed to violate model choice reversal.  

\section{Asymptotic Limitations on Ensemble Aggregated--Prediction Instability}
\noindent The discussion of the link between model capacity and ensemble aggregation instability highlights the role played by having surjective, unconstrained learners in creating violations of our aggregation prediction/choice properties.  In turn, this suggests that restrictions to the ranges of individual models or to the mapping from feature space to scores may eliminate portions of the scoring space that violate our properties, and so promote stability in ensemble decision making.  Following this reasoning, we conjecture that as the data available for training grows arbitrarily large this will act as an implicit restriction on the diversity of the models in the ensemble and thus the ranges and mappings of individual models, and so minimize the risk of prediction instability.  

To formalize this intuition, we draw upon the notion of machine learning consistency.  We suppose that there is a true conditional distribution over $Y$, say $p(Y = y|X = x)$, that reflects the probability of $y\in Y$ given full knowledge of the underlying data population (rather than the sample available for training).  Given this true conditional distribution, it is well known \citep[e.g.][]{biau-2008} that the Bayes classifier minimizes the risk of making incorrect predictions or choices, and so consistency is typically defined in a machine learning context as convergence in probability to the prediction or choice made by the Bayes classifier as the size of the data sample available for training grows.

Our version of consistency is similar, if slightly stronger:

\begin{definition}
Say that a sequence of model scoring functions $s_j(x|D_n)$ is \textbf{fully consistent} if for all $\varepsilon>0$
\begin{align*}
0 &= \lim_{n\to\infty} p\left(\sum_{y,\yp}\int\limits_X\hspace{-2mm}\begin{array}{l}
I\{s_j(x,y|D_n)\leq s_j(x,\yp|D_n)\}\\
\hspace{0mm}\times I\{p(y|x)>p(\yp|x)\}p(x) dx
\end{array}\hspace{-2mm}\geq\varepsilon\right).
\end{align*}
\end{definition}
That is, a model scoring function is fully consistent if the probability that it disagrees with the pairwise ordering of the true conditional distribution is zero in the limit of arbitrarily large amounts of training data.\footnote{Note that this is  stronger than the usual definition of machine learning consistency.  The usual definition only requires agreement between the model and Bayes classifier on the choice from the full set $Y$.  Full consistency is sufficient but not necessary for this in the sense that if $s_j(x|D_n)$ is fully consistent then a classifier defined on it in the usual way will converge to the Bayes classifier.}

If all models in an ensemble are fully consistent then we have the following result:

\begin{theorem}Suppose that $s_j(x|D_n)$ is fully consistent for all $j$.  Then for almost all classification problems as $n\to\infty$ there exists a non-degenerate ensemble choice aggregator $\Ce$ that satisfies insertion/deletion consistency, ensemble unanimity, and model choice reversal for all $x\in X$ with probability one.\label{thm:consistency}
\end{theorem}
Consistency of the individual models in the ensemble means that asymptotically it is possible to find ensemble aggregation procedures that satisfy definitions  \ref{ICDC}, \ref{NECA}, \ref{score_unanimity}, and \ref{CR}.\footnote{The ``almost all" in the statement of the theorem refers to the true conditional distribution -- for any finite set $Y$ the set of distributions with ties between the elements of $Y$ is measure zero relative to the full set of distributions over $Y$.  Also, although we have proved this result with asymptotic convergence of the models in the ensemble to the true conditional distribution in mind, agreement of the models in the ensemble with the order derived from the true conditional distribution is not necessary to eliminate prediction or choice instability.  Rather, the key is simply agreement of the models with one another full stop, whatever distribution they ultimately agree on.}  Specifically, soft voting, which suffers from violations to model choice reversal in finite samples, will satisfy all properties with consistent models operating in asymptotic conditions.  Prediction or choice instability is in part a finite sample issue.

Finally, it is worth noting that full consistency at the level of the ensemble is not strong enough; a consistent ensemble may still exhibit prediction or choice instability even in asymptotic conditions.  To find ensemble choice aggregators that will not violate our properties we need all models in the ensemble to be fully consistent.  Put another way, diversity among the models within the ensemble, the very thing that may be desirable for generalization error minimization, creates the prediction or choice instability we analyze here.

\section{Conclusion}

\noindent In this note we have analyzed a set of four desirable properties on ensemble information use and prediction stability and shown that at least one must be violated in any machine learning ensemble, establishing a fundamental tradeoff between information use and prediction instability.  Any ensemble that uses the insights of multiple models will make predictions or choices in a counterintuitive or unstable fashion at least some of the time.  The risk of prediction instability is created by aggregation of multiple models and heightened by the presence of high capacity individual models within the ensemble.  Since three of these four properties are fairly easy to satisfy the last--model choice reversal--must fail in many common use cases thus introducing `discontinuities' into the predictions and decisions made by ensembles.  Small changes in underlying model scores can lead to large, substantive changes in ensemble predictions and choices; the ensemble can change its prediction or choice even if none of the underlying models do and there is no guarantee that similar feature space inputs will be treated similarly.  Finally, consistency of the individual models in the ensemble implies that it is possible to find aggregation procedures that eliminate the prediction or choice instability analyzed here in asymptotic conditions.

We observe that theorem \ref{thm:aggregation_instability} as presented here shows that all ensembles will behave badly some of the time with respect to the prediction stability properties we have analyzed.  It strikes us that the requirements on the model scoring function range and the surjectivity of $s(x)$ are likely much stronger assumptions in regression problems than in classification problems; though theorem \ref{thm:aggregation_instability} holds in a technical sense in regression we speculate that the risk of running into model choice reversal violations is lower in regression problems.  It also seems intuitive that the risk of prediction or choice instability depends on the number of models in the underlying ensemble and on the relationship between the size of the data sample used to train the ensemble and the population it is drawn from.  Accordingly we aim to extend the analysis here to investigate asymptotic ensemble behavior when the number of models in the ensemble grows large.  We conjecture that the Borda count method of aggregation may satisfy definitions \ref{ICDC}, \ref{NECA}, and \ref{score_unanimity} with a slight weakening of model choice reversal in the asymptotic case of many models.  Finally, we conjecture that the argument developed here would apply to using a data sampling process like bootstrapping or $k$-fold cross validation to choose among multiple model architectures during training.

\section*{Acknowledgment}
\noindent This work was possible by the exceptional generosity of the PPC Foundation, Inc. and through support by my colleagues Christopher Taylor and Sheila Ross. Additionally, we would like to thank RJ Nowling, David Andrzejewski, and the audience of the Medical College of Wisconsin SEAWINDS Data Science Symposium for helpful comments, suggestions, and feedback.  Finally, we would like to thank Ernesto Guerra Vallejos for suggestions leading to theorem \ref{thm:consistency}.

\section*{Appendix}

\begin{lemma}If $\Ce$ satisfies transitivity then:
\begin{enumerate} 
\item $y\in\C{z}$ and $\yp=\C[\{\yp,\ypp\}]{z}$ implies that $y=\C[\{y,\ypp\}]{z}$;
\item $y=\C{z}$ and $\yp\in\C[\{\yp,\ypp\}]{z}$ implies that $y=\C[\{y,\ypp\}]{z}$;
\item $y=\C{z}$ and $\yp=\C[\{\yp,\ypp\}]{z}$ implies that $y=\C[\{y,\ypp\}]{z}$.
\end{enumerate}\label{transitivity}
\end{lemma}
\begin{proof}Suppose that $\Ce$ satisfies transitivity and take any $z\in S(Y)$.  To see part 1) assume $y\in\C{z}$ and $\{\yp\}=\C[\{\yp,\ypp\}]{z}$.  Then by transitivity $y\in\C[\{y,\ypp\}]{z}$.  Suppose that $\ypp\in\C[\{y,\ypp\}]{z}$.  But then by transitivity $\ypp\in\C[\{\yp,\ypp\}]{z}$, a contradiction.  To see part 2) assume $\{y\}=\C{z}$ and $\yp\in\C[\{\yp,\ypp\}]{z}$.  Then by transitivity $y\in\C[\{y,\ypp\}]{z}$.  Suppose that $\ypp\in\C[\{y,\ypp\}]{z}$.  But then by transitivity $\yp\in\C{z}$, a contradiction.  To see, part 3) assume that $\{y\}=\C{z}$ and $\{\yp\}=\C[\{\yp,\ypp\}]{z}$ but that $\ypp\in \C[\{y,\ypp\}]{z}$.  Then by part 1) we have $\ypp\in\C[\{\yp,\ypp\}]{z}$, a contradiction.
\end{proof}

\noindent \textbf{Proof of Lemma \ref{DC_IC_transitivity}.}
\begin{proof}Sufficiency.  Suppose that $\Ce$ is transitive.  Take any $z\in S(Y)$, any $Y_1^*,Y_2^*\subseteq Y$ such that $Y_1^*\subseteq Y_2^*$.  Suppose that $Y_1^*\bigcap\C[Y_2^*]{z}\neq\emptyset$.  Take any $y\in\C[Y_1^*]{z}$; we wish to show that $y\in Y_1^*\bigcap\C[Y_2^*]{z}$. Clearly $y\in Y_1^*$.  Suppose that $y\notin \C[Y_2^*]{z}$.  Take any $\yp\in Y_1^*\bigcap\C[Y_2^*]{z}$.  Then $\yp\in\C[Y_2^*]{z}$ implies that $\yp\in\C[]{z}$, and $y\in \C[Y_1^*]{z}$ implies that $\{y,\yp\}=\C{z}$.  Since $y\notin\C[Y_2^*]{z}$ it follows that there exists $\ypp\in Y_2^*$ such that $\{\ypp\} = \C[\{y,\ypp\}]{z}$.  But then by transitivity and Lemma \ref{transitivity} it must be that $\{\ypp\} = \C[\{\yp,\ypp\}]{z}$, a contradiction.  Next take any $y\in Y_1^*\bigcap\C[Y_2^*]{z}$.  Then there does not exist $\yp\in Y_2^*$ such that $\{\yp\}=\C{z}$.  Thus there does not exist $\yp\in Y_1^*$ such that $\{\yp\}=\C{z}$ and so $y\in \C[Y_1^*]{z}$.  It follows that $\C[Y_1^*]{z} = Y_1^*\bigcap\C[Y_2^*]{z}$ and so definition \ref{ICDC} holds.

Necessity.  Suppose that $\Ce$ satisfies definition \ref{ICDC} but that transitivity does not hold.  Then there exists $y,\yp,\ypp$ such that $y\in\C{z}$ and $\yp\in\C[\{\yp,\ypp\}]{z}$, yet $\{\ypp\}=\C[\{y,\ypp\}]{z}$.  There are four cases to consider.  

\textbf{Case 1}: $\{y,\yp\}=\C{z}$, $\{\yp,\ypp\}=\C[\{\yp,\ypp\}]{z}$, and $\{\ypp\}=\C[\{y,\ypp\}]{z}$.  By definition of $\Ce$ we have $\C[\{y,\yp,\ypp\}]{z} = \{\yp,\ypp\}$.  But then we have $\{y,\yp\}\subseteq \{y,\yp,\ypp\}$ and $\{y,\yp\}\bigcap\C[\{y,\yp,\ypp\}]{z}\neq\emptyset$ and yet $\C{z} = \{y,\yp\}\neq \{\yp\}=\{y,\yp\}\bigcap\C[\{y,\yp,\ypp\}]{z}$, a violation of definition \ref{ICDC} and so a contradiction.  

\textbf{Case 2}: $\{y,\yp\}=\C{z}$, $\{\yp\}=\C[\{\yp,\ypp\}]{z}$, and $\{\ypp\}=\C[\{y,\ypp\}]{z}$.  By definition of $\Ce$ we have $\C[\{y,\yp,\ypp\}]{z} = \{\yp\}$.  But then we have $\{y,\yp\}\subseteq \{y,\yp,\ypp\}$ and $\{y,\yp\}\bigcap\C[\{y,\yp,\ypp\}]{z}\neq\emptyset$ and yet $\C{z} = \{y,\yp\}\neq \{\yp\}=\{y,\yp\}\bigcap\C[\{y,\yp,\ypp\}]{z}$, a violation of definition \ref{ICDC} and so a contradiction.  

\textbf{Case 3}: $\{y\}=\C{z}$, $\{\yp,\ypp\}=\C[\{\yp,\ypp\}]{z}$, and $\{\ypp\}=\C[\{y,\ypp\}]{z}$.  By definition of $\Ce$ we have $\C[\{y,\yp,\ypp\}]{z} = \{\ypp\}$.  But then we have $\{\yp,\ypp\}\subseteq \{y,\yp,\ypp\}$ and $\{\yp,\ypp\}\bigcap\C[\{y,\yp,\ypp\}]{z}\neq\emptyset$ and yet $\C[\{\yp,\ypp\}]{z} = \{\yp,\ypp\}\neq \{\ypp\}=\{\yp,\ypp\}\bigcap\C[\{y,\yp,\ypp\}]{z}$, a violation of definition \ref{ICDC} and so a contradiction.  

\textbf{Case 4}: $\{y\}=\C{z}$, $\{\yp\}=\C[\{\yp,\ypp\}]{z}$, and $\{\ypp\}=\C[\{y,\ypp\}]{z}$.  But then $\C[\{y,\yp,\ypp\}]{z} = \emptyset$, a contradiction.
\end{proof}

\noindent\textbf{Proof of Lemma \ref{mcr_same_choice}}.
\begin{proof}Take any $z,\zpr\in S(Y)$.  It suffices to consider two cases.  First suppose that $\C{z}=\{y,y^{\prime}\}$ but that without loss of generality $\C{\zpr}=\{y\}$.  Then there is an immediate contradiction of model choice reversal, since there does not exist $j$ such that $z_{j,d(y^{\prime})}>z_{j,d(y)}$ and $\zpr_{j,d(y^{\prime})}<\zpr_{j,d(y)}$, and so the claim holds.  Next suppose without loss of generality that $\C{z}=\{y\}$ but that $y^{\prime}\in \C{\zpr}$.  But then again there is an immediate contradiction of model choice reversal, since there does not exist $j$ such that $z_{j,d(\yp)}<z_{j,d(y)}$ and $\zpr_{d,d(\yp)}>\zpr_{j,d(y)}$ and so the claim holds.  Thus $\C{z} = \C{\zpr}$.
\end{proof}

\noindent \textbf{Proof of theorem \ref{thm:aggregation_instability}.}
\begin{proof} Suppose that $|Y|\geq 3$ and that by contradiction $\Ce$ satisfies definitions \ref{def:unconstrained_learner}, \ref{NECA}, \ref{ICDC}, \ref{score_unanimity}, and \ref{CR}.  We will prove (1) in a series of steps, following the basic structure of Fey (2014).  To begin, note that by Lemma \ref{DC_IC_transitivity} $\Ce$ satisfies definition \ref{def_of_trans}.  Say that $j$ is \textit{decisive} for $y$ over $y^{\prime}$ if $z_{j,d(y)}>z_{j,d(\yp)}$ implies that $\{y\}=\C{z}$.  Take any $y,y^{\prime}\in Y$, $y\neq y^{\prime}$.  Finally, note that all $z$ postulated below are available in $S(Y)$ by $1,\hdots,m$ being unconstrained learners.

\noindent\textbf{Step 1: identify a critical model }$\mathbf{j^*}$.  Consider the following $s,\spr\in S(Y)$:
\begin{align*}
\z[1]{y} & > \z[1]{\yp} > \hdots\\
& \vdots\\
\z[m]{y} & > \z[m]{\yp} > \hdots
\end{align*}
and:
\begin{align*}
\zp[1]{\yp} & > \zp[1]{y} > \hdots\\
 & \vdots\\
\zp[m]{\yp} & > \zp[m]{y} > \hdots
\end{align*}
By ensemble unanimity, $\{y\}=\C{z}$ and $\{y^{\prime}\}=\C{\zpr}$.  Now, consider $z^j$ such that:
\begin{align*}
\zj[1]{\yp}{j} & > \zj[1]{y}{j} > \hdots\\
& \vdots\\
\zj[j]{\yp}{j} & > \zj[j]{y}{j} > \hdots\\
\zj[j+1]{y}{j} & > \zj[j+1]{\yp}{j} > \hdots\\
& \vdots\\
\zj[m]{y}{j} & > \zj[m]{\yp}{j} > \hdots
\end{align*}
Define $j^*$ such that $y\in \C{z^{j^*-1}}$ and $y^{\prime}\in \C{z^{j^*}}$.

\noindent\textbf{Step 2: for all }$\mathbf{y^{\prime\prime}\neq y,y^{\prime}}$ \textbf{we have} $\mathbf{j^*}$ \textbf{decisive for} $\mathbf{y^{\prime}}$ \textbf{over} $\mathbf{y^{\prime\prime}}$.  To see this, take any $y^{\prime\prime}$ and consider $z\in S(Y)$ such that:
\begin{align}
\z[1]{\yp} & > \z[1]{\ypp} > \z[1]{y} \hdots\nonumber\\
& \vdots \nonumber\\
\z[j^*-1]{\yp} & > \z[j^*-1]{\ypp} > \z[j^*-1]{y} \hdots\nonumber\\
\z[j^*]{y} & > \z[j^*]{\yp} > \z[j^*]{\ypp} \hdots\nonumber\\
 & \vdots\nonumber \\
\z[m]{y} & > \z[m]{\yp} > \z[m]{\ypp} \hdots\hdots\label{profile_3}
\end{align}
Note that $\z[j]{y}>\z[j]{\yp}$ if and only if $\z[j^*-1]{y}>\z[j^*-1]{\yp}$ for all $j$ and so by Lemma \ref{mcr_same_choice} $y\in \C{z}$.  By ensemble unanimity $\{y^{\prime}\}=\C[\{y^{\prime},y^{\prime\prime}\}]{z}$.  Thus by Lemma \ref{transitivity} $\{y\}=\C[\{y,y^{\prime\prime}\}]{z}$.  Next consider the set of scores $S^*\subseteq S(Y)$ such that for all $z^*\in S^*$:
\begin{align}
\zj[1]{\yp}{*}\backslash \zj[1]{\ypp}{*} & > \zj[1]{y}{*} > \hdots\nonumber\\
 & \vdots \nonumber\\
\zj[j^*]{\yp}{*} & > \zj[j^*]{y}{*} >\zj[j^*]{\ypp}{*} \hdots\nonumber\\
\zj[j^*+1]{y}{*} & > \zj[j^*+1]{\yp}{*}\backslash \zj[j^*+1]{\ypp}{*} > \hdots\nonumber\\
 & \vdots \nonumber\\
\zj[m]{y}{*} & > \zj[m]{\yp}{*}\backslash \zj[m]{\ypp}{*} > \hdots\label{profile_4}
%
\end{align}
where $\z[j]{a}\backslash \z[j]{b}$ denotes any of $\z[j]{a} > \z[j]{b}$, $\z[j]{a} < \z[j]{b}$ or $\z[j]{a} = \z[j]{b}$.  Note that for any $z^*\in S^*$ we have $\zj[j]{y}{*}>\zj[j]{\yp}{*}$ if and only if $\zj[j]{y}{j^*}>\zj[j]{\yp}{j^*}$ for all $j$ and so by Lemma \ref{mcr_same_choice} we have $y^{\prime}\in \C{z^*}$.  Moreover, for any $z^*\in S^*$ we have that $\zj[j]{y}{*}>\zj[j]{\ypp}{*}$ if and only if the same holds in \ref{profile_3} for all $j$ and so Lemma \ref{mcr_same_choice} also implies that $\{y\}=\C[\{y,y^{\prime\prime}\}]{z^*}$.  Thus by Lemma \ref{transitivity} we have $\{y^{\prime}\}=\C[\{y^{\prime},y^{\prime\prime}\}]{z^*}$.  But now note that for any $z$ such that $\zj[j^*]{\yp}{}> \zj[j^*]{\ypp}{}$ there exists $z^{*}\in S^*$ such that $\zj[j]{\yp}{}>\zj[j]{\ypp}{}$ if and only if $\zj[j]{\yp}{*}>\zj[j]{\ypp}{*}$.  By Lemma \ref{mcr_same_choice} we have $\{y^{\prime}\}=\C[\{y^{\prime},y^{\prime\prime}\}]{z}$.  We conclude that $j^*$ is decisive for $y^{\prime}$ over $y^{\prime\prime}$.

\noindent\textbf{Step 3: for all $\mathbf{y^{\prime\prime}\neq y,y^{\prime}}$ we have $\mathbf{j^*}$ decisive for $\mathbf{y}$ over $\mathbf{y^{\prime\prime}}$}.  Consider the set of ensemble scores $S^*\subseteq S(Y)$ such that for all $z^*\in S^*$:
\begin{align}
\zj[1]{y}{*}\backslash \zj[1]{\ypp}{*} & > \zj[1]{\yp}{*} > \hdots\nonumber\\
 & \vdots \nonumber\\
\zj[j^*]{y}{*} & > \zj[j^*]{\yp}{*} > \zj[j^*]{\ypp}{*} > \hdots\nonumber\\
\zj[j^*+1]{y}{*}\backslash \zj[j^*+1]{\ypp}{*} & > \zj[j^*+1]{\yp}{*} > \hdots\nonumber\\
 & \vdots\nonumber \\
\zj[m]{y}{*}\backslash \zj[m]{\ypp}{*} & > \zj[m]{\yp}{*} > \hdots\label{profile_5}
\end{align}
By step (2) we have that $j^*$ is decisive for $y^{\prime}$ over $y^{\prime\prime}$ and so $\{y^{\prime}\}=\C[\{y^{\prime},y^{\prime\prime}\}]{z^*}$.  By ensemble unanimity we have $\{y\}=\C{z^*}$.  By Lemma \ref{transitivity} we have $\{y\}=\C[\{y,y^{\prime\prime}\}]{z^*}$.  But now note that for any $z$ such that $\zj[j^*]{y}{} > \zj[j^*]{\ypp}{}$ there exists $z^{*}\in S^*$ such that $\zj[j]{y}{}>\zj[j]{\ypp}{}$ if and only if $\zj[j]{y}{*} >\zj[j]{\ypp}{*}$ for all $j$ and so by Lemma \ref{mcr_same_choice} we have that $\{y\}=\C[\{y,y^{\prime\prime}\}]{z}$.  We conclude that $j^*$ is decisive for $y$ over $y^{\prime\prime}$.

\noindent\textbf{Step 4: for all $\mathbf{y^{\prime\prime}\neq y,y^{\prime}}$ we have $\mathbf{j^*}$ decisive for $\mathbf{y^{\prime\prime}}$ over $\mathbf{y}$.}  Consider $s$ such that:
\begin{align}
\zj[1]{\yp}{} & > \zj[1]{\ypp}{} > \zj[1]{y}{} \hdots\nonumber\\
 & \vdots \nonumber\\
\zj[j^*]{\ypp}{} & > \zj[j^*]{y}{} > \zj[j^*]{\yp}{}\hdots\nonumber\\
 & \vdots \nonumber\\
\zj[m]{\ypp}{} & > \zj[m]{y}{} > \zj[m]{\yp}{} \hdots\label{profile_6}
\end{align}
Note that $\zj[j]{y}{} > \zj[j]{\yp}{}$ iff $\zj[j]{y}{j^*-1} >\zj[j]{y}{j^*-1}$ for all $j$ and so by Lemma \ref{mcr_same_choice} we have $y\in \C{z}$.  By ensemble unanimity $\{y^{\prime\prime}\}=\C[\{y,y^{\prime\prime}\}]{z}$ and so by Lemma \ref{transitivity} we have $\{y^{\prime\prime}\}=\C[\{y^{\prime},y^{\prime\prime}\}]{z}$.  Consider the set of ensemble scores $S^*\subseteq S(Y)$ such that for all $z^*\in S^*$:
\begin{align}
\zj[1]{\yp}{*} & > \zj[1]{y}{*} \backslash \zj[1]{\ypp}{*} > \hdots\nonumber\\
 & \vdots \nonumber\\
\zj[j^*]{\ypp}{*} & > \zj[j^*]{\yp}{*} > \zj[j^*]{y}{*} \geq \hdots\nonumber\\
\zj[j^*+1]{y}{*} \backslash \zj[j^*+1]{\ypp}{*} & > \zj[j^*+1]{\yp}{*} > \hdots\nonumber\\
 & \vdots\nonumber \\
\zj[m]{y}{*} \backslash \zj[m]{\ypp}{*} & > \zj[m]{\yp}{*} > \hdots\nonumber
\end{align}
Note that for any $z^*\in S^*$ we have $\zj[j]{y}{*} >\zj[j]{\yp}{*}$ iff $\zj[j]{y}{j^*} >\zj[j]{\yp}{j^*}$ and so by Lemma \ref{mcr_same_choice} we have $y^{\prime}\in\C{z^*}$. Also, $\zj[j]{\yp}{*} > \zj[j]{\ypp}{*}$ iff the same holds in \ref{profile_6} and so by Lemma \ref{mcr_same_choice} we have $\{y^{\prime\prime}\}=\C[\{y^{\prime},y^{\prime\prime}\}]{z^*}$.  By Lemma \ref{transitivity} it follows that $\{y^{\prime\prime}\}=\C[\{y,y^{\prime\prime}\}]{z^*}$.  But now note that for any $z$ such that $\zj[j^*]{\ypp}{} > \zj[j^*]{y}{}$ there exists $z^*\in S^*$ such that $\zj[j]{\ypp}{} >\zj[j]{y}{}$ iff $\zj[j]{\ypp}{*} > \zj[j]{y}{*}$ for all $j$ and so by Lemma \ref{mcr_same_choice} we have $\{\ypp\}=\C[\{y,\ypp\}]{z}$. We conclude that $j^*$ is decisive for $y^{\prime\prime}$ over $y$.

\noindent\textbf{Step 5: for all $\mathbf{y^{\prime\prime}\neq y,y^{\prime}}$ we have $\mathbf{j^*}$ decisive for $\mathbf{y^{\prime\prime}}$ over $\mathbf{y^{\prime}}$.}  Consider the set of ensemble scores $S^*\subseteq S(Y)$ such that for all $z^*\in S^*$:
\begin{align}
\zanysec[1]{*}{y}{\yp}{\ypp} > \hdots\nonumber\\
 & \vdots \nonumber\\
\zline[j^*]{*}{\ypp}{y}{\yp}> \hdots\nonumber\\
\zanysec[j^*+1]{*}{y}{\yp}{\ypp} > \hdots\nonumber\\
 & \vdots \nonumber\\
 \zanysec[m]{*}{y}{\yp}{\ypp} > \hdots\label{profile_8}
\end{align}
By step (4) we have that $j^*$ is decisive for $y^{\prime\prime}$ over $y$, and so $\{\ypp\}=\C[\{y,\ypp\}]{z}$.  By ensemble unanimity $\{y\}=\C[\{y,\yp\}]{z^*}$.  By Lemma \ref{transitivity} we have $\{\ypp\}=\C[\{\yp,\ypp\}]{z^*}$. But now note that for any $z$ such that $\zj[j^*]{\ypp}{} > \zj[j^*]{\yp}{}$ there exists $z^{*}\in S^*$ such that $\zj[j]{\ypp}{} > \zj[j]{\yp}{}$ iff $\zj[j]{\ypp}{*} > \zj[j]{\yp}{*}$ for all $j$ and so by Lemma \ref{mcr_same_choice} we have that $\{\ypp\}=\C[\{\yp,\ypp\}]{z}$. We conclude that $j^*$ is decisive for $y^{\prime\prime}$ over $y^{\prime}$.

\noindent\textbf{Step 6: $\mathbf{j^*}$ decisive for $\mathbf{y}$ over $\mathbf{y^{\prime}}$ and for $\mathbf{y^{\prime}}$ over $\mathbf{y}$.}  Consider the set of ensemble scores $S^*\subseteq S(Y)$ such that for all $s^*\in S^*$:
\begin{align}
\zj[1]{y}{*} \backslash \zj[1]{\yp}{*} & >  \hdots\nonumber\\
 & \vdots \nonumber\\
\zj[j^*]{y}{*} & > \zj[j^*]{\ypp}{*} > \zj[j^*]{\yp}{*} > \hdots\nonumber\\
\zj[j^*+1]{y}{*} \backslash \zj[j^*+1]{\yp}{*} & > \hdots\nonumber\\
 & \vdots\nonumber \\
\zj[m]{y}{*} \backslash \zj[m]{\yp}{*} & > \hdots\label{profile_9}
\end{align}
By step (5) we have that $j^*$ is decisive for $y^{\prime\prime}$ over $y^{\prime}$ and so $\{\ypp\}=\C[\{\yp,\ypp\}]{z^*}$.  By step (3) we have that $j^*$ is decisive for $y$ over $y^{\prime\prime}$ and so $\{y\}=\C[\{y,\ypp\}]{z^*}$.  By Lemma \ref{transitivity} we have $\{y\}=\C{z^*}$.  But now note that for any $z$ such that $\zj[j^*]{y}{} > \zj[j^*]{\yp}{}$ there exists $z^*\in S^*$ such that $\zj[j]{y}{} > \zj[j]{\yp}{}$ iff $\zj[j]{y}{*} > \zj[j]{\yp}{*}$ and so by Lemma \ref{mcr_same_choice} we have $\{y\}=\C{z}$.  We conclude that $j^*$ is decisive for $y$ over $y^{\prime}$.  The argument for $j^*$ decisive for $y^{\prime}$ over $y$ is the same with the two reversed, citing step (4) and step (2).

\noindent\textbf{Step 7: $\mathbf{j^*}$ is decisive for any $a\in Y$ over any $b\in Y$.} Take any $a\neq y,y^{\prime}$ and any $b\neq y,y^{\prime}$ and consider the set of scores $S^*\subseteq S(Y)$ such that for all $z^*\in S^*$:
\begin{align}
\zj[1]{a}{*} \backslash \zj[1]{b}{*} & > \zj[1]{y}{*} \backslash \zj[1]{\yp}{*} \hdots\nonumber\\
 & \vdots \nonumber\\
\zj[j^*]{a}{*} & > \zj[j^*]{y}{*} > \zj[j^*]{\yp}{*} > \zj[j^*]{b}{*} \hdots\nonumber\\
 & \vdots \nonumber\\
\zj[m]{a}{*} \backslash \zj[m]{b}{*} & > \zj[m]{y}{*} \backslash \zj[m]{\yp}{*} \hdots\label{profile_10}
\end{align}
By step (4) we have that $j^*$ is decisive for $a$ over $y$, by step (6) we have that $j^*$ decisive for $y$ over $y^{\prime}$, and by step (2) we have that $j$ is decisive for $y^{\prime}$ over $b$.  Thus by Lemma \ref{transitivity} we have $\{a\}=\C[\{a,b\}]{z^*}$.  But now note that for any $z$ such that $\zj[j^*]{a}{} > \zj[j^*]{b}{}$ there exists $z^{*}\in S^*$ such that $\zj[j]{a}{} > \zj[j]{b}{}$ iff $\zj[j]{a}{*} > \zj[j]{b}{*}$ for all $j$, and so by Lemma \ref{mcr_same_choice} we have $\{a\}=\C[\{a,b\}]{z}$. Thus $j^*$ decisive for $a$ over $b$ and since $a$ and $b$ were arbitrary, there exists $j$, namely $j^*$, such that for all $y,y^{\prime}\in Y$ if $\zj[j^*]{y}{} > \zj[j^*]{\yp}{}$ then $\{y\}=\C{z}$.  Thus definition \ref{NECA} is violated, a contradiction.  This establishes (1).

By the forgoing, we have established that any $\Ce$ that satisfies definition \ref{NECA} at least one of definitions \ref{ICDC}, \ref{score_unanimity}, or \ref{CR}.  Thus, if $\Ce$ that satisfies definition \ref{NECA} there exists $S^*\subseteq S(Y)$ such that at least one of definitions \ref{ICDC}, \ref{score_unanimity}, or \ref{CR} are violated on $S^*$.  By $s(\cdot)$ surjective, for all $s^*\in S^*$ there exists $x\in X$ such that $s(x)=s^*$.  Thus, define:
\begin{align*}
X^*\equiv\bigcup_{s^*\in S^*}\left\{x\in X|s(x) = s^*\right\}.
\end{align*}
This establishes (2).

\end{proof}

Given a strict ordering $P\in\mathbf{R}$ define:
\begin{align*}
\overline{W}(Y,P) = \left\{z\in\mathbb{R}^{|Y|}\left|\mbox{for all }y,y^{\prime}\in Y \atop{y\hspace{1mm} P\hspace{1mm}y^{\prime}\mbox{ iff }z_{d(y)}> z_{d(y^{\prime})}}\right.\right\}.
\end{align*}
\begin{lemma}Consider a restriction on $S(Y)$ such that for any $z\in S(Y)$ there exists $P\in\mathbf{R}$ such that $z_j\in \overline{W}(Y,P)$ for all $j$.  Then if $\Ce$ satisfies ensemble unanimity it respects model choice reversal.\label{lemma:unanimity_MCR}
\end{lemma}
\begin{proof}Assume that the hypotheses of the lemma hold and consider any $\Ce$, any $z,\zpr\in S(Y)$, and any $y,\yp\in Y$ such that $y = \C{z}$ and $\yp\in\C{\zpr}$.  By assumption, there exists $P,P^{\prime}\in\mathbf{R}$ such that for all $j$ we have $z_j\in \overline{W}(Y,P)$ and $\zpr\in \overline{W}(Y,P^{\prime})$.  Thus, it must be that for all $j$ we have $z_{j,d(y)}>z_{j,d(\yp)}$, else by unanimity $\yp \in \C{z}$, a contradiction.  Suppose that there does not exist $j$ such that $z_{j,d(y)}>z_{j,d(\yp)}$ and $\zpr_{j,d(y)}<\zpr_{j,d(\yp)}$.  Then $\zpr\in \overline{W}(Y,P^{\prime})$ implies that $\zpr_{j,d(y)}>\zpr_{j,d(\yp)}$ for all $j$.  By ensemble unanimity $y=\C{\zpr}$, a contradiction.
\end{proof}

\noindent \textbf{Proof of theorem \ref{thm:consistency}.}
\begin{proof}Fix a finite set of label values $Y = \{y_1,y_2,\hdots\}$ and a true conditional distribution $p(Y=y|X=x)$.  Generically, for any $x\in X$, there exists a strict ordering $P_x\in\mathbf{R}$ so that $\{p(y|x)\}_{y\in Y}\in\overline{W}(Y,P_x)$.  By $s_j(x|D_n)$ fully consistent it follows that $s_j(x)\in\overline{W}(Y,P_x)$ for all $j$ and all $x$ in the limit as $n\to\infty$ with probability one.  Now, consider the soft voting ensemble choice aggregator:
\begin{align*}
C(Y,s(x)) = \argmax_{y\in Y}\left\{\frac{1}{m}\sum_{j}s_j(x,y)\right\}.
\end{align*}
Clearly this non-degenerate and satisfies ensemble unanimity.  It is also transitive and so by Lemma \ref{DC_IC_transitivity} satisfies insertion/deletion consistency.  Finally, by Lemma \ref{lemma:unanimity_MCR} it also satisfies model choice reversal for all $x\in X$ with probability one.
\end{proof}


\end{document}